\def\year{2022}\relax
\documentclass[letterpaper]{article} 
\usepackage{aaai22}  
\usepackage{times}  
\usepackage{helvet}  
\usepackage{courier}  
\usepackage[hyphens]{url}  
\usepackage{graphicx} 
\usepackage{natbib}  
\usepackage{caption} 
\DeclareCaptionStyle{ruled}{labelfont=normalfont,labelsep=colon,strut=off} 
\usepackage{algorithm}
\usepackage{algorithmic}

\usepackage{amssymb}
\usepackage{color}
\usepackage{amsthm}
\usepackage{amsmath}
\usepackage{multirow} 
\usepackage{paralist, tabularx}
\usepackage{subcaption}
\usepackage{multirow}
\usepackage{booktabs}

\newtheorem{theorem}{Theorem}

\newtheorem{assumption}{Assumption}
\newtheorem{proposition}{Proposition}

\newcommand{\m}[1]{\mathbf{#1}}
\title{Deterministic and Discriminative Imitation (D2-Imitation): \\
Revisiting Adversarial Imitation for Sample Efficiency}
\author {
    Mingfei Sun\textsuperscript{\rm 1, 2}, 
    Sam Devlin\textsuperscript{\rm 2}, 
    Katja Hofmann\textsuperscript{\rm 2}, 
    Shimon Whiteson\textsuperscript{\rm 1}
}
\affiliations {
    \textsuperscript{\rm 1}University of Oxford \\
    \textsuperscript{\rm 2}Microsoft Research \\
    \{mingfei.sun, shimon.whiteson\}@cs.ox.ac.uk, \{sam.devlin, katja.hofmann\}@microsoft.com
}

\begin{document}

\maketitle

\begin{abstract}
Sample efficiency is crucial for imitation learning methods to be applicable in real-world applications. 
Many studies improve sample efficiency by extending adversarial imitation to be off-policy 
regardless of the fact that these off-policy extensions could either change the original objective or involve complicated optimization.  
We revisit the foundation of adversarial imitation and propose an off-policy sample efficient approach that requires no adversarial training or min-max optimization.
Our formulation capitalizes on two key insights: (1) the similarity between the Bellman equation and the stationary state-action distribution equation allows us to derive a novel temporal difference (TD) learning approach; 
and (2) the use of a deterministic policy simplifies the TD learning. 
Combined, these insights yield a practical algorithm, Deterministic and Discriminative Imitation (D2-Imitation), 
which operates by first partitioning samples into two replay buffers and then learning a deterministic policy via off-policy reinforcement learning. 
Our empirical results show that D2-Imitation is effective in achieving good sample efficiency, outperforming several off-policy extension approaches of adversarial imitation on many control tasks.
\end{abstract}

\section{Introduction}
We consider a specific imitation learning task: 
given a limited number of expert demonstrations, 
 learn an optimal policy in a simulator without any access to the expert policy or reinforcement signals of any kind.
One of the important approaches for this problem setting, a.k.a \emph{Apprenticeship Learning}~\cite{abbeel2004apprenticeship}, is adversarial imitation, 
which learns a policy by minimizing the divergence between 
the stationary state-action distribution induced by the policy and the expert distribution given in the demonstrations~\cite{ho2016generative,kostrikov2018discriminator,ke2019imitation,ghasemipour2020divergence}.  
As the two distributions are given in the form of empirical samples, 
adversarial imitation has to roll out the policy to collect on-policy samples and then estimates the empirical divergence by training a discriminator~\cite{ho2016generative}.
Those on-policy samples are then discarded immediately after each policy update.  The resulting sample inefficiency is a crucial limitation, ruling out many real-world applications, where collecting interactions is often slow (e.g., physical robotics) or expensive (e.g., autonomous driving).

Some methods improve sample efficiency by relaxing the on-policy requirement in adversarial imitation, 
e.g., using off-policy samples to train the discriminator and policy~\cite{kostrikov2018discriminator,sasaki2018sample}.
Doing so, however, changes the original divergence minimization objective, 
hence providing no guarantee that the learned policy recovers the expert distribution~\cite{kostrikov2019imitation,sun2021softdice}.
Furthermore, the non-stationary rewards generated by the discriminator make actor-critic algorithms converge slowly (or even diverge)~\cite{sutton2018reinforcement} 
and could demand significantly more samples to learn a good policy~\cite{schulman2015high,fujimoto2018addressing}. 
Other methods translate the on-policy distribution matching problem into an off-policy optimization, 
which also circumvents the need to learn rewards~\cite{kostrikov2019imitation,sun2021softdice}. 
This type of approach involves a non-convex min-max optimization procedure that requires well-tuned regularization \cite{gulrajani2017improved},  
limiting its practical application. 

\begin{figure}
    \centering
    \includegraphics[width=1.0\linewidth]{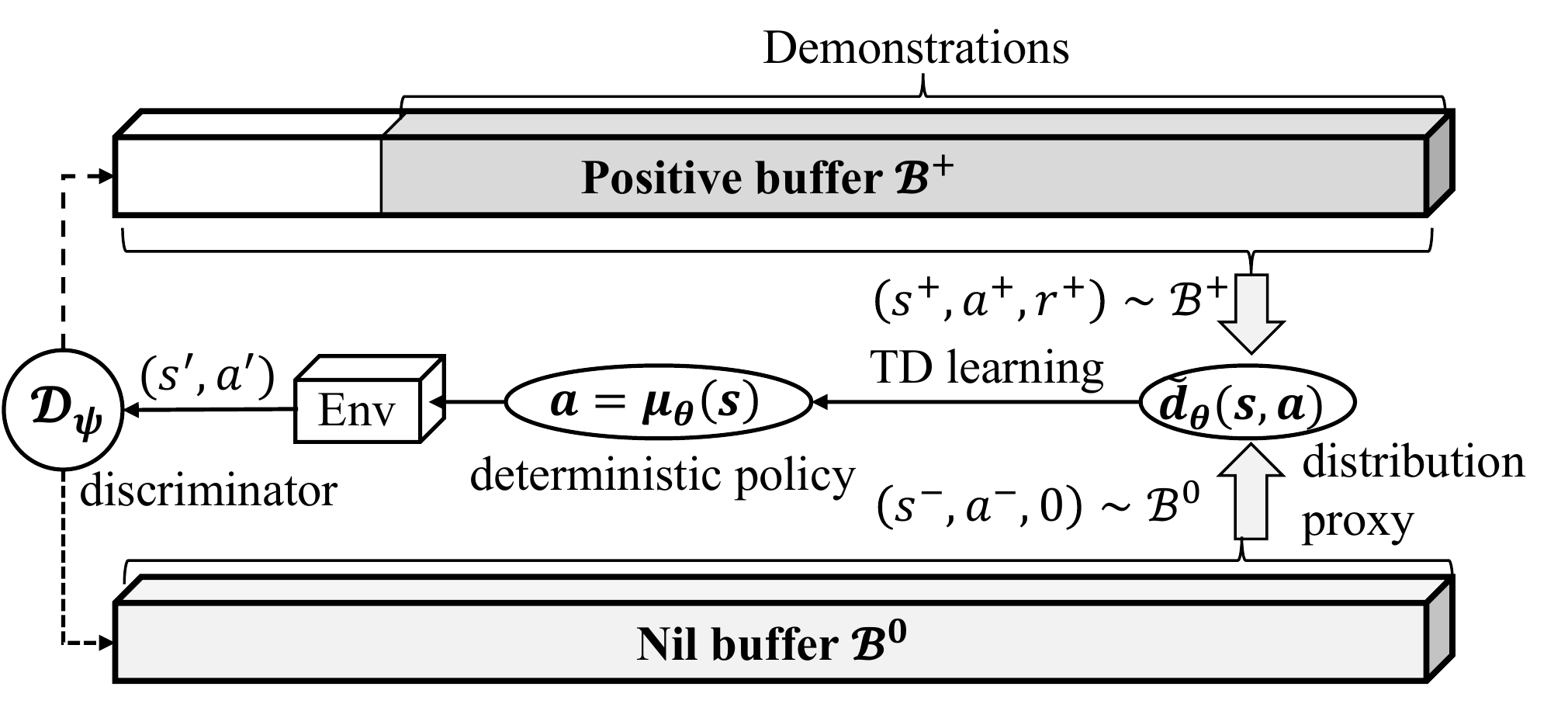}
    \caption{Overview of the D2-Imitation learning\footnotemark.}
    \label{fig:d2-imitation}
\end{figure}
\footnotetext{Code repository: \url{https://github.com/mingfeisun/d2-imitation}.}

In this paper, we revisit the foundation of adversarial imitation and propose an off-policy learning approach, 
Deterministic and Discriminative Imitation (D2-Imitation), that involves no adversarial training or non-convex min-max optimization. 
D2-Imitation capitalizes on two key insights: (1) 
the similarity between the Bellman equation and the stationarity equation of the state-action distribution from the formulation of adversarial imitation, 
and (2) the benefit of using a deterministic policy in the stationarity equation. 
In particular, instead of defining the divergence for distribution matching as in~\citet{ho2016generative,ghasemipour2020divergence}, 
we derive a temporal difference (TD) learning approach from the stationarity equation, which directly learns a proxy for the expert state-action distribution. 
Furthermore, we show that the use of a deterministic policy simplifies TD learning and yields a practical imitation learning algorithm, 
which operates by first partitioning samples into two replay buffers via a discriminator, 
and then learning a deterministic policy via off-policy RL, 
as presented in Figure~\ref{fig:d2-imitation}.
More specifically,  D2-Imitation first trains a discriminator $\mathcal{D}_{\phi}$ with state-action pairs from demonstrations (i.e., positive samples) and randomly generated samples. 
Whenever a sample is collected by rolling out the policy in simulation, 
the discriminator determines which buffer, $\mathcal{B}^{+}$ or $\mathcal{B}^{0}$, the sample should be put in. 
D2-Imitation then uniformly samples from the two replay buffers, assigns positive rewards to the samples from $\mathcal{B}^{+}$ and zero to those from $\mathcal{B}^{0}$, 
and uses those samples to update the proxy estimate of state-action distribution via an off-policy RL algorithm. 
We show that the notion of partitioning samples into two groups and assigning constant rewards follows theoretically from the use of a deterministic policy. 
Our empirical results show that D2-Imitation is effective in achieving good sample efficiency, outperforming many adversarial imitation approaches on different benchmark tasks with demonstrations from either deterministic or stochastic policies.

\section{Background}
\paragraph{Preliminaries}
We consider an infinite-horizon Markov decision process (MDP) with a finite state space $\mathcal{S}$, 
a finite action space $\mathcal{A}$, 
a transition kernel $p: \mathcal{S}\times\mathcal{A}\times\mathcal{S}\rightarrow (0, 1]$, 
a reward function $r: \mathcal{S}\times\mathcal{A}\rightarrow\mathbb{R}$, 
a discount factor $\gamma \in [0, 1)$, 
and an initial state distribution $p_0$ from which $s_0$ is sampled. 
An agent then interacts with its environment in discrete time steps as follows:
at each time step $t$ the agent observes environment state $s_t\in\mathcal{S}$, 
takes an action $a_t\in\mathcal{A}$ and transitions to a successor state $s_{t+1}\in\mathcal{S}$, 
also denoted as $s^\prime$ if the subscript is omitted. 
The agent's behavior is defined by a policy $\pi$, 
which maps states to a probability distribution over the action: $\pi: \mathcal{S}\rightarrow \mathcal{P}(\mathcal{A})$. 
A policy is deterministic if the probability distribution degenerates to a probability mass point.
We use $\mu$ to denote a deterministic policy and $\mu(s)$ to denote the action for $s$. 
The return from a state is defined as the sum of discounted future reward $R_t=\sum_{i=t}^{\infty}\gamma^{(i-t)}r(s_i, a_i)$.
The return depends on the actions chosen, and therefore on policy $\pi$.
A policy is optimal if it maximizes the expected return from the start distribution $J(\pi)=\mathbb{E}_{s_0\sim p_0, \pi_t}[R_0]$. 
In imitation learning, the agent needs to learn an optimal policy from expert samples $(s_E, a_E)$ that are generated by an optimal policy $\pi_E$, 
assuming the reward function $r$ is unknown and no reinforcement signals of any kind are available. 

\paragraph{Deterministic policy gradients}
For a parametrized policy $\pi_{\theta}$, its parameters $\theta$ can be updated in the direction of the performance gradient $\nabla_{\theta}J(\pi_{\theta})$. 
When the policy is deterministic, the parameters can be updated through the \textit{deterministic policy gradient}~\cite{silver2014deterministic},
\begin{equation}\label{equ:deterministic-pg}
\nabla_{\theta}J(\pi_\theta) = \mathbb{E}_{s\sim d^\pi} \big[ \nabla_a Q^\pi(s, a)|_{a=\pi(s)}\nabla_{\theta}\pi_{\theta}(s) \big], 
\end{equation}
where $d^\pi$ is the state distribution induced by $\pi$ (the formal definition is given in the next section), 
and $Q^\pi(s, a)$, also called the action value function, is the expected return after taking $a_t$ in state $s_t$ and thereafter following $\pi$: $Q^{\pi}(s_t, a_t) = \mathbb{E}_{r_{i\geq t}, s_{i\geq t}\sim p, a_{i\geq t}\sim\pi}\big[ R_t| s_t, a_t \big]$, 
To learn $Q^{\pi}(s, a)$, one can update the critic recursively based on the Bellman equation~\cite{sutton2018reinforcement}:
\begin{equation}\label{equ:bellman-equation}
Q^{\pi}(s, a) = \mathbb{E}_{r, s^\prime\sim p}\big[ r(s, a) + \gamma \mathbb{E}_{a^\prime\sim \pi} [Q^\pi(s^\prime, a^\prime)] \big]. 
\end{equation}
For a large state space, the value function can be estimated with a differentiable neural network $Q_{\phi}(s, a)$, with parameters $\phi$. 
To stabilize the training, in deep $Q$-learning~\cite{mnih2015human}, the network is updated by temporal difference learning with a target network $Q_{\phi^\prime}(s, a)$ to maintain a fixed objective $y$ over multiple updates:
$y = r + \gamma Q_{\phi^\prime}(s^\prime, a^\prime), \quad a^\prime\sim\pi_{\theta^\prime}(s^\prime),$
where the actions are selected from a target actor network $\pi_{\theta^\prime}$.

\section{Analysis of Stationary Distribution}\label{sec:d2-imitation}
We first revisit adversarial imitation 
and analyze the pivotal idea that underpins this family of approaches: the stationarity equation for state-action distributions. 
We show the similarity between the stationarity equation and the Bellman equation, 
and propose a temporal difference learning approach to learn a proxy for the state-action distribution. 
Finally, we consider a practical algorithm that learns a deterministic policy in an off-policy manner.

\subsection{Distribution Matching in Adversarial Imitation}\label{sec:duality}
Adversarial imitation capitalizes on the idea that the stationary distribution induced by the learned policy should match the empirical state-action distribution presented in the demonstrations~\cite{ho2016generative}. 
To characterize such idea, we first describe the duality between a policy and its stationary state-action distribution. 
For a policy $\pi\in\Pi$, define its stationary state-action distribution as  $d(s, a)=(1-\gamma)\sum_{t=0}^{\infty}\gamma^t d_t(s, a)$, 
where $d_t(s, a)=P \big( s_t=s, a_t=a|s_0\sim p_0, a_i\sim\pi(\cdot|s_i), s_{i+1}\sim p(\cdot|s_i, a_i), \forall i<t \big)$.
A basic result from~\citet{puterman2014markov} is that the set of valid state-action distributions $\Omega\triangleq\{d_{\pi}: \pi\in\Pi\}$ can be characterized by a feasible set of affine constraints, 
if $p_0(s)$ is the starting state distribution ($p_0(s)>0 \quad \forall s\in\mathcal{S}$), then
\begin{multline}\label{equ:feasible-constraint}
\Omega=\big\{d: \sum_{a^\prime} d(s^\prime, a^\prime) = (1-\gamma) p_0(s^\prime) + \\
\gamma\sum_{s, a}p(s^\prime|s, a)d(s, a) \quad d(s, a)\geq 0 \quad \forall (s^\prime, a^\prime)
\big\}.
\end{multline}
Furthermore, there is a one-to-one mapping for $\Pi$, $\Omega$:
\begin{proposition}\label{prop:one-one-correspondence}
\cite{syed2007game} If $d(s, a)$ is feasible by Equation~\eqref{equ:feasible-constraint}, then $d$ is the state-action distribution for $\pi\triangleq d(s, a)/\sum_{a}d(s, a)$, and $\pi$ is the only policy whose state-action distribution is $d$. 
\end{proposition} 
This proposition also holds for deterministic policies.
In particular, a solution $d$ is a \emph{basic feasible solution} (BFS) of affine constraints if it cannot be expressed as a nontrivial convex combination of any other feasible solutions. 
If $d(s, a)$ is a BFS of Equation~\eqref{equ:feasible-constraint}, then $\pi$ is deterministic and vice versa. 
Thus, for a BFS $d(s, a)$ and its corresponding deterministic policy $\mu$, we always have $d(s, a)=0$ if $a\neq\mu(s)$ and $d(s, \mu(s))>0$~\cite{puterman2014markov}.

Proposition~\ref{prop:one-one-correspondence} forms the foundation of distribution matching approaches for adversarial imitation. 
Specifically, since the stationary state-action distribution uniquely describes a policy, 
one can learn a policy $\pi$ by minimizing the difference between two distributions: the state-action distribution $d_{\pi}(s, a)$ induced by the policy $\pi$ and the state-action distribution $d_{E}(s, a)$ in the demonstrations~\cite{ho2016generative,kostrikov2018discriminator,ke2019imitation,ghasemipour2020divergence}.  
In adversarial imitation, the distribution difference is estimated via a discriminator, 
which requires rolling out the policy (i.e., generator) to simulate an MDP process and to collect on-policy samples from $d_{\pi}(s, a)$ whenever the policy is updated in training~\cite{ho2016generative}.
Those on-policy samples are discarded immediately after each iteration~\cite{kostrikov2018discriminator,kostrikov2019imitation}. 
We take a different perspective by formulating a novel learning method from Equation~\eqref{equ:feasible-constraint}, 
without resorting to any divergence minimization or repeatedly requiring on-policy samples.

\subsection{Distribution Matching by TD Learning}
At first glance, the stationarity equation~\eqref{equ:feasible-constraint} resembles the Bellman equation for action-value functions. 
Specifically, multiplying by $\pi(a^\prime|s^\prime)$ on both sides of~\eqref{equ:feasible-constraint}, we have
\begin{multline}\label{equ:feasible-constraint-sa}
d(s^\prime, a^\prime) = (1-\gamma) p_0(s^\prime)\pi(a^\prime|s^\prime) + \\
\gamma\sum_{s, a}\pi(a^\prime|s^\prime) p(s^\prime|s, a)d(s, a), 
\end{multline}
which holds for $d\geq 0, \forall (s^\prime, a^\prime)\in\mathcal{S}\times\mathcal{A}$. 
With this new formulation, $d(s, a)$ can be reinterpreted as the action-value function and $(1-\gamma)p_0(s^\prime)\pi(a^\prime|s^\prime)$ as the expected reward over the state $s$ (denote $r(s, a) = (1-\gamma)p_0(s)\pi(a|s)$). 
The Bellman equation~\eqref{equ:bellman-equation} bootstraps the estimate of the current action value $Q(s, a)$ on the next ones $(s^\prime , a^\prime)$, 
i.e., leveraging successor samples.  By contrast, the stationary equation~\eqref{equ:feasible-constraint-sa} operates backwards by bootstrapping the current value $d(s^\prime, a^\prime)$ on the previous ones $(s, a)$, 
i.e., leveraging predecessor samples. 
In practice, it is difficult to directly obtain predecessor samples in model-free learning~\cite{liu2018breaking}. 
We thus consider an alternative equation as follows, 
\begin{multline}\label{equ:feasible-constraint-successor}
\tilde{d}(s, a) = (1-\gamma) p_0(s)\pi(a|s) + \\
\gamma\sum_{s^\prime, a^\prime}\pi(a^\prime|s^\prime) p(s^\prime|s, a) \tilde{d}(s^\prime, a^\prime), 
\end{multline}
which modifies the original stationarity equation by considering successor samples, rather than predecessor ones.
We use $\tilde{d}(s, a)$ to note that $\tilde{d}(s, a)$ defined by this equation is \emph{not} strictly a stationary state-action distribution for any $\pi$. 
However, we show below that the BFSs of Equation~\eqref{equ:feasible-constraint-successor} are the same as that of Equation~\eqref{equ:feasible-constraint-sa}. 
Denote state-action distribution as $\m{d} \in\mathbb{R}^{|\mathcal{S}|\times|\mathcal{A}|}$, 
initial state-action distribution as $\m{d}_0 \in\mathbb{R}^{|\mathcal{S}|\times|\mathcal{A}|}$, 
where $d_0(s, a) = p_0(s)\pi(a|s)$, 
and state-action transition probability as $\m{P}_{\pi}\in\mathbb{R}^{(|\mathcal{S}|\times|\mathcal{A}|)\times (|\mathcal{S}|\times|\mathcal{A}|)}$, 
where $P_{\pi}\big( (s, a), (s^\prime, a^\prime)\big)=\pi(a^\prime|s^\prime)p(s^\prime|s, a)$. 
We make the following assumption: 
\begin{assumption}
The Markov chain induced by $\pi$ is ergodic and $(\m{I} - \gamma \m{P}_{\pi}^\top)^{-1}$ exists. 
\end{assumption}
The first part of this assumption can be easily fulfilled in the real world as long as the problem to be considered has a recurring structure~\cite{levin2017markov}. 
The second part is to ensure that $\tilde{d}(s, a)$ is well-defined~\cite{yu2015convergence}. 
We then have the following, 
\begin{theorem}\label{theo:bfs-invariant}
For any strictly positive function $p_0(s)$, the affine constraints defined in Equations \eqref{equ:feasible-constraint-sa} and \eqref{equ:feasible-constraint-successor} have the same set of BFSs. 
\end{theorem}
\begin{proof}
Denote BFS as $d^B(s, a)$ and its corresponding policy as $\mu$.
Since $d^B$ is a BFS, we have $d^B(s, a)=0$ if $a\neq\mu(s)$ following from Proposition~\ref{prop:one-one-correspondence}. 
We only need to show that $d^B(s, a)$ remains feasible for any strictly positive $p_0(s)$. 

The affine constraints given in Equations~\eqref{equ:feasible-constraint-sa} and~\eqref{equ:feasible-constraint-successor} can be represented in the matrix form 
$\m{d} = \gamma \m{P}_{\pi}^\top\m{d} + (1-\gamma) \m{d}_0$ 
and $\tilde{\m{d}} = \gamma \m{P}_{\pi}\tilde{\m{d}} + (1-\gamma) \m{d}_0$. 
Thus we have $\m{d} = (1-\gamma) (\m{I} - \gamma \m{P}_{\pi}^\top)^{-1} \m{d}_0$
and $\tilde{\m{d}} = (1-\gamma) (\m{I} - \gamma \m{P}_{\pi})^{-1} \m{d}_0$.
Note that $(\m{I} - \gamma \m{P}_{\pi})^{-1}$ also exists given that $(\m{I} - \gamma \m{P}_{\pi}^\top)^{-1}$ exists. 
Thus for both affine constraints, the BFSs are determined as follows: 
$\m{d}^B = (\m{I} - \gamma \m{P}^\top_{\pi})^{-1} \m{p}_0$, 
and $\tilde{\m{d}}^B = (\m{I} - \gamma \m{P}_{\pi})^{-1} \m{p}_0$. 
For strictly positive function $p_0(s)$, we have $\m{p}_0 > 0$.
Also, since $\gamma\in [0, 1)$ and the Markov chain induced by $\pi$ is ergodic (also iperiodic), 
we have $\forall s$, $\big[(\m{I} - \gamma \m{P}^\top_{\pi})^{-1} \m{p}_0 \big]_{\big(s,\mu(s)\big)} > 0$, 
and $\big[(\m{I} - \gamma \m{P}_{\pi})^{-1} \m{p}_0 \big]_{\big(s,\mu(s)\big)} > 0$.
Thus $d^B\big(s, \mu(s) \big)$ is feasible for both affine constraints. 
Therefore, the set of BFSs for these two affine constraints are the same. 
\end{proof}

Furthermore, as Equation~\eqref{equ:feasible-constraint-successor} relies on the successor samples, 
we can thus construct the following temporal difference (TD) update rule to learn $\tilde{d}(s, a)$, 
\begin{multline}\label{equ:td-learning}
\tilde{d}(s_t, a_t) \leftarrow \tilde{d}(s_t,a_t) + \\
\alpha_t \big[ r(s_t, a_t) + \gamma \tilde{d}(s_{t+1},a_{t+1}) - \tilde{d}(s_t, a_t) \big], 
\end{multline} 
where $r(s_t, a_t) = (1-\gamma) p_0(s_t)\pi(a_t|s_t)$ 
and $\{\alpha_t\}$ is a deterministic positive non-increasing sequence satisfying the Robbins-Monro condition~\cite{robbins1951stochastic}, 
i.e., $\sum_t\alpha_t = \infty$ and $\sum_t \alpha_t^2 < \infty$. 
The convergence analysis of such TD learning is the same as that for the Bellman operator~\cite{yu2015convergence} and is thus omitted here.

\section{D2-Imitation}
Theorem~\ref{theo:bfs-invariant} and the TD update in Equation~\eqref{equ:td-learning} suggest a TD approach to imitation learning:
if the reward $r(s, a)$ is set as $(1-\gamma)p_0(s)\mu_E(a|s)$, where $\mu_E(a|s)$ is the deterministic expert policy, 
one can recover the BFSs of Equation~\eqref{equ:feasible-constraint-sa} for the expert policy $\mu_E$ via the TD update rule defined in Equation~\eqref{equ:td-learning}. 
However, in Apprenticeship Learning, 
we have access only to state-action samples, 
i.e., $\{(s, a): a\sim\pi_E(\cdot|s)\}$. 
Thus, the reward $r(s, a) = (1-\gamma) p_0(s)\mu_E(a|s)$ in Equation~\eqref{equ:td-learning} is generally unknown. 
In the following, we show that the reward configuration can be simplified for a deterministic policy, which yields a practical algorithm.

For a deterministic policy $\mu_E$, the reward $r(s, a)$ in Equation~\eqref{equ:td-learning} can be further simplified. 
Specifically, we have $r(s, a) = (1-\gamma) p_0(s)\mu_E(a|s) = 0$ if $a\neq\mu_E(s)$, 
and $r(s, a)= (1-\gamma) p_0(s)$ if $a=\mu_E(s)$.
According to Theorem~\ref{theo:bfs-invariant}, 
the set of BFSs, corresponding to the deterministic policy, 
does not rely on $p_0(s)$ as long as it is strictly positive for all states.
We can thus redefine $r(s, a)$  as: 
$r_d(s, a) = 1$ for $a=\mu_E(s)$, and $r_d(s, a)=0$ for $a\neq\mu_E(s)$. 
Namely, we can partition all state-action pairs into two groups (corresponding to two replay buffers) based on $\mu_E(s)$: 
a \textit{positive buffer} $\mathcal{B}^{+}$, which contains all $(s, a)$ such that $a=\mu_E(s)$, 
and a \textit{nil buffer} $\mathcal{B}^{0}$, which contains all $(s, a)$ such that $a\neq\mu_E(s)$. 
As we have no access to the expert policy $\mu_E$ and no further queries can be made, 
we have to learn a discriminator that can tell whether $a=\mu_E(s)$ or $a\neq\mu_E(s)$ for a given $(s, a)$. 
Since the demonstrations are given as a finite set of samples $\{(s_E, a_E): a_E\sim\pi(\cdot|s_E)\}$, 
we can  train such a discriminator as follows:
for each $s_E \in \{(s_E, a_E): a_E\sim\pi_E(\cdot|s)\}$, generate some random actions $a_{\text{rnd}}$ to form samples $(s_E, a_{\text{rnd}})$, 
assign negative labels to the random samples and positive labels to the expert samples, 
and train a discriminator via supervised learning with those labeled samples.

Furthermore, following from Theorem~\ref{theo:bfs-invariant}, 
we can learn BFSs for $\mu_E$ via TD learning by updating $\tilde{d}(s, a)$ with samples that are uniformly sampled from two replay buffers. 
For continuous action spaces, an expert policy $\mu_{E}$ parametrized with $\theta_{\mu}$ could be directly learned via deep deterministic policy gradients~\cite{lillicrap2015continuous},
as follows: 
\begin{equation}\label{equ:deterministic-update}
\nabla_{\theta_\mu}J(\mu_{\theta}) = \mathbb{E}_{\substack{s\sim d_{\beta}(s) \\a=\mu(s)}}\big[ \nabla_a \tilde{d}(s, a|\theta_{\tilde{d}})\nabla_{\theta_\mu}\mu(s|\theta_\mu)  \big],
\end{equation}
where $\theta_{\tilde{d}}$ is the parameters of $\tilde{d}$ and $\tilde{d}(s, a|\theta_{\tilde{d}})$ is updated based on Equation~\eqref{equ:td-learning}.
To stabilize the policy training, we also create copies $\theta_\mu^\prime$ and $\theta_{\tilde{d}}^\prime$ for the actor networks $\theta_{\mu}$ and the critic networks $\theta_{\tilde{d}}$.
For exploration, we can construct an exploration policy $\mu^\prime$ by adding Gaussian noise to the target policy $\mu$, 
i.e., $\mu^\prime(s_t) = \mu(s_t|\theta_\mu) + \mathcal{N}$, 
where $\mathcal{N}$ can be chosen to suit the environment,
e.g., based on domain knowledge of the environment.

Based on the theoretical and conceptual insights developed above, 
we now present Deterministic and Discriminative Imitation (D2-Imitation), 
an off-policy imitation learning method that trains a deterministic policy with a discriminator.
During training, the algorithm proceeds as follows: 
It first populates $\mathcal{B}^{+}$ with all demonstration state-action pairs $(s_E, a_E)$, 
and then generates random actions $a_{\text{rnd}}$ for each state $s_E$ in $\mathcal{B}^{+}$. 
In the next step, it trains the discriminator parametrized with $\psi$, $\mathcal{D}_{\psi}$, with positive samples from $\mathcal{B}^{+}$ and random samples $(s_E, a_{\text{rnd}})$.
In the policy training, 
whenever a sample $(s, a)$ is generated by rolling out the policy in simulation, 
the discriminator then determines which buffer the sample should be put in. 
We set a probability threshold $P_{th}$ such that $(s, a)$ is put into $\mathcal{B}^{+}$ only when the probability given by $\mathcal{D}_{\psi}$ exceeds $P_{th}$.
Otherwise, the sample $(s, a)$ is put into $\mathcal{B}^{0}$. 
In the following steps, D2-Imitation assigns a constant positive reward $1$ to all samples in $\mathcal{B}^{+}$, 
and zero reward to those in $\mathcal{B}^{0}$, 
and updates the distribution proxy $\theta_{\tilde{d}}$ and the policy $\theta_{\mu}$ based on Equation~\eqref{equ:td-learning} and Equation~\eqref{equ:deterministic-update} respectively.
D2-Imitation is detailed in Algorithm~\ref{algo:d2-imitation}. 

\begin{algorithm}[tb]
\caption{D2-Imitation}\label{algo:d2-imitation}
\begin{algorithmic}[1] 
\STATE Initialize $\mathcal{B}^{0}$ to be empty, $\mathcal{B}^{+}$ with $\{(s_{E}, a_{E}) \}$. 
\STATE Generate random actions $a_{\text{rnd}}$ for state $s_E\in\mathcal{B}^{+}$. 
\STATE Train $\mathcal{D}_{\psi}$ with samples from $\mathcal{B}^{+}$ and $\{(s_{E}, a_{\text{rnd}})\}$.
\FOR{$t=1$ to $T$}
\STATE Collect samples $(s, a)$ by policy $\mu_\theta$.
\STATE Predict probability $p$ of $(s, a)$ being positive with $\mathcal{D}_{\psi}$.
\IF{$p \geq P_{th}$}
\STATE Store transition $(s, a)$ in $\mathcal{B}^{+}$.
\ELSE
\STATE Store transition $(s, a)$ in $\mathcal{B}^{0}$.
\ENDIF
\STATE Sample $\{(s_i^+, a_i^+)\}$ from $\mathcal{B}^{+}$, $\{(s_i^0, a_i^0)\}$ from $\mathcal{B}^{0}$.
\STATE Assign reward $1$ to $\{(s_i^+, a_i^+) \}$ and zero to $\{(s_i^0, a_i^0) \}$.
\STATE Update $\theta_{\tilde{d}}$, $\theta_{\mu}$ based on Equation~\eqref{equ:td-learning}, \eqref{equ:deterministic-update} respectively.
\IF{$t \mod d$}
\STATE Update target networks. 
\ENDIF
\ENDFOR
\end{algorithmic}
\end{algorithm}

Compared to adversarial imitation approaches, D2 differs in the following perspectives:
First, D2 does not assume the expert demonstrations are from a stationary state-action distribution. 
In other words, D2 only requires the samples to be from a conditional distribution $\pi_E(a|s)$ 
rather than from $d_E(s, a)$. 
Consequently, D2 learns a discriminator which conditions on the state, 
whereas adversarial imitation approaches train a discriminator for joint state-action pairs. 
Second, D2 leverages TD learning to directly approximate a proxy to the BFSs, 
while adversarial imitation approaches rely on density ratio estimation, which involves variational estimate of a divergence~\cite{nguyen2010estimating} that generates non-stationary rewards in training.

The above analysis for D2-Imitation assumes that the expert policy is deterministic such that BFSs are the same for Equations~\eqref{equ:feasible-constraint-sa} and~\eqref{equ:feasible-constraint-successor}. 
If the expert policy is stochastic, 
the notion of learning a deterministic policy in effect considers the setting that only one feasible action should be learned. 
For apprenticeship learning in large or continuous state spaces, e.g., continuous control tasks,
it is exceedingly unlikely to have the exact same state twice in the demonstrations.
Thus, learning deterministic policies in apprenticeship learning effectively amounts to distilling one-one mappings from demonstrations. 
In the next section, we show that D2-Imitation still works well empirically even when the given demonstrations are generated by stochastic policies.

\section{Experiments}

We evaluate D2-Imitation on several popular benchmarks for continuous control. 
We first use four physics-based control tasks: Swimmer, Hopper, Walker2d and HalfCheetah,
ranging from low-dimensional control tasks to difficult high-dimensional ones. 
Each task comes with a true cost function~\cite{brockman2016openai}. 
We first generate expert demonstrations by running state-of-the-art RL algorithms, 
including Soft Actor Critic (SAC)~\cite{haarnoja2018soft}, 
Proximal Policy Optimization (PPO)~\cite{schulman2017proximal}, 
Deep Deterministic Policy Gradient (DDPG)~\cite{lillicrap2015continuous}, 
and Twin-Delayed DDPG (TD3)~\cite{fujimoto2018addressing}, 
on the true cost functions. 
We train five different policies using independent seeds for three million time steps. 
To collect demonstrations, 
we roll out each expert policy to gather 20 trajectories
and for each environment we have 100 trajectories altogether. 
Then, to evaluate imitation performance, 
we sample a certain number (5, 10, 15, or 20) of trajectories as demonstrations.
This procedure repeats for other domains, including classic controls and Box2D simulation. 
For the critic and actor networks, 
we use a two layer MLP with ReLu activations ($\tanh$ activation for the last layer in the actor network) and two hidden units (256+256). 
For the discriminator we use the same architecture as~\citet{ho2016generative}: 
a two layer MLP with 100 hidden units and $tanh$ activations. 
These design choices have been empirically shown to be best suited to control tasks. 
We train all networks with Adam~\cite{kingma2014adam} with a learning rate of $10^{-3}$. 
The discriminator is pre-trained for 1000 iterations with batch size 256. 
We set the probability threshold heuristically to $0.8$ (after a sweep from $0.7$ to $0.95$) as it empirically works well across all domains. 
In positive buffer $\mathcal{B}^+$, we force the off-policy samples to account for only a small portion ($~25\%$), 
and the majority is still from demonstrations. 
The implementation of off-policy TD follows that of TD3~\cite{fujimoto2018addressing}, with a double $Q$-net.

\subsection{Comparing to Adversarial Imitation Methods}
\begin{figure}
    \centering
    \includegraphics[width=0.9\linewidth]{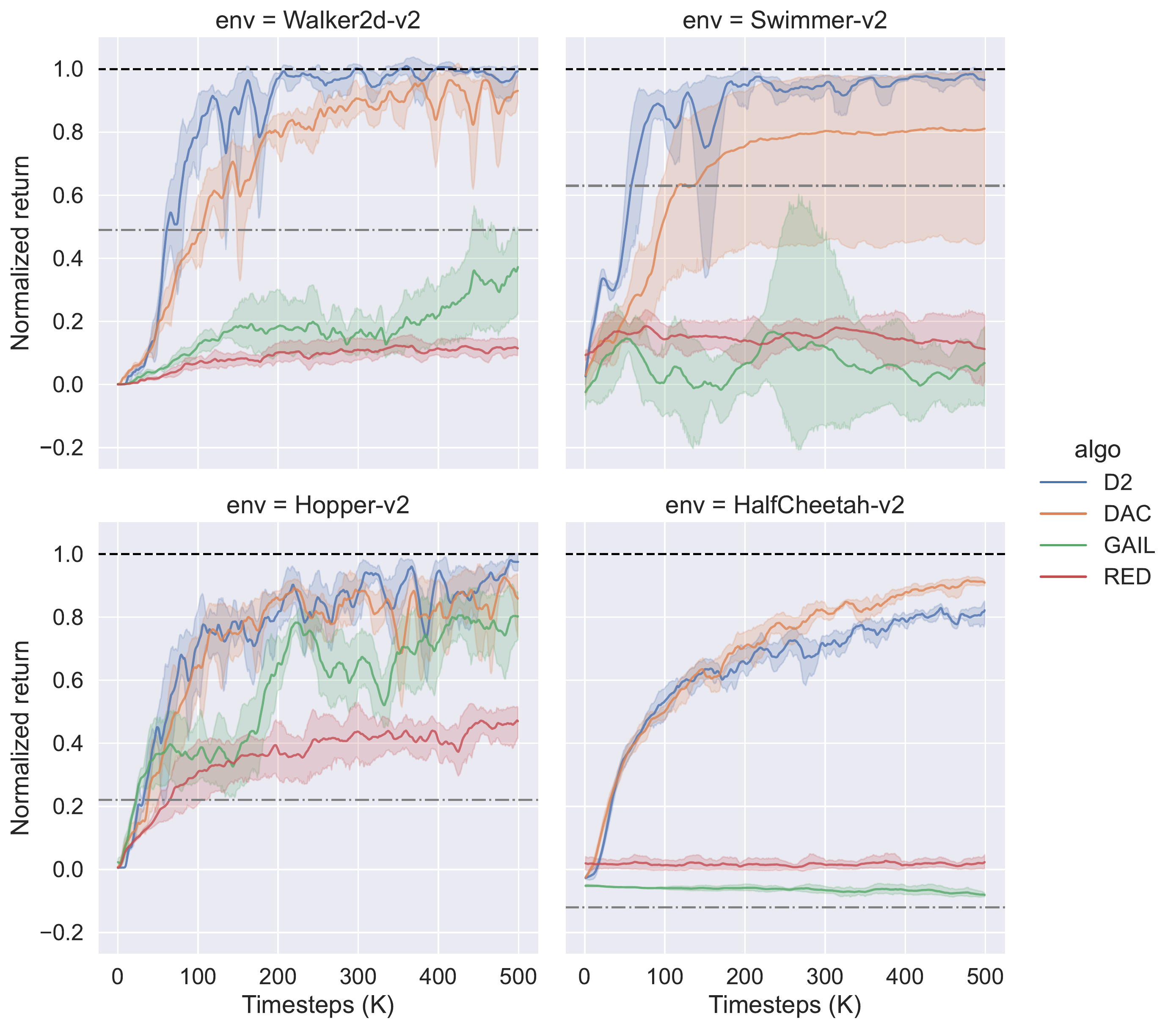}
    \caption{Comparisons of algorithms using 20 SAC demonstration trajectories; $y$-axis: normalized episodic returns ($1.0$ for an expert policy, as marked by the dark dashed line; BC performance is indicated by the gray dotted line). }
    \label{fig:comparison-with-adversarial}
\end{figure}

\begin{figure*}
    \centering
    \includegraphics[width=0.85\linewidth]{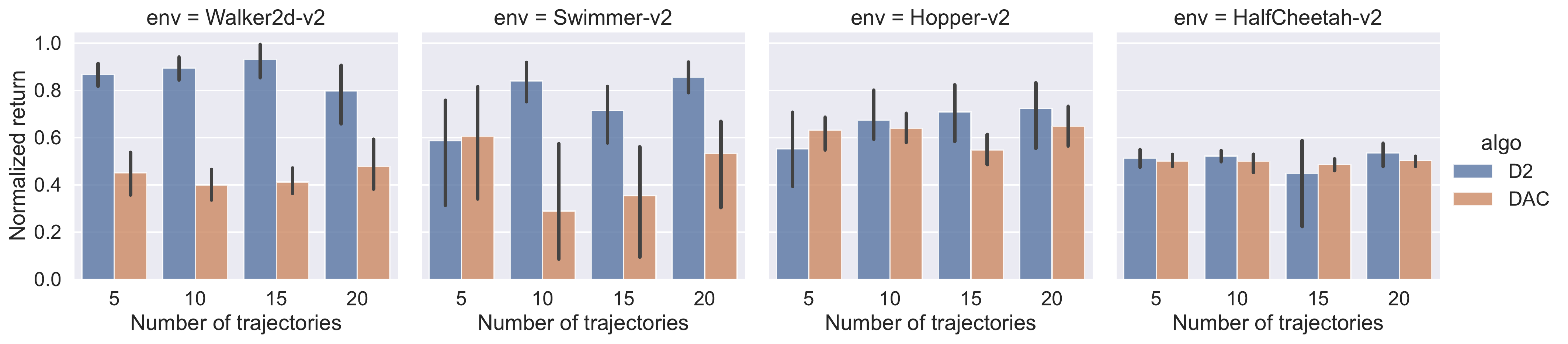}
    \caption{Performance comparisons using different numbers of demonstration trajectories (5, 10, 15, and 20);
    $y$-axis: normalized episodic returns after training for \textit{100K} environment interactions ($1.0$ indicates the returns of an expert policy).}
    \label{fig:comparison-num-traj}
\end{figure*}

We first compare our algorithm to two extensions of adversarial imitation: 
discriminator-actor-critic (DAC)~\cite{kostrikov2018discriminator}, 
and Random Expert Distillation (RED)~\cite{wang2019random}, 
Particularly, DAC addresses sample efficiency by reusing past experiences through a replay buffer, 
while RED does so by stabilizing RL training with a fixed reward function, 
rather than the non-stationary one in GAIL.
We use the original implementation of DAC and RED.\footnote{DAC: \url{https://github.com/google-research/google-research/tree/master/dac}; RED: \url{https://github.com/RuohanW/RED}.}
DAC also adopts the following practice to improve sample efficiency: 
a replay buffer and the deterministic policy updated via TD3 algorithm. 
We also compare the proposed algorithm with a variant of GAIL in which the policy is pre-trained with behavior cloning (BC) to reduce the sample complexity, as used by~\citet{ho2016generative}.
We tune the number of pre-training iterations for different environments to achieve stable performance. 
We also found that too many iterations hinder learning of GAIL as the behavioral cloning may  overfit to the demonstration data. 
We perform evaluation using 10 random seeds. 
For each seed, we log average episodic returns from the simulation environment in training. 
Like~\citet{ho2016generative} and~\citet{kostrikov2018discriminator}, we plot returns normalized by expert performance. 
We compute means over all seeds and visualize standard deviations.

Figure~\ref{fig:comparison-with-adversarial}
shows that D2 is much more sample efficient than the pre-trained GAIL and RED in all tested environments with SAC demonstrations. 
GAIL's performance is worse than reported in~\cite{ho2016generative} as the original GAIL was trained for $\geq$2M timesteps for Hopper and Swimmer, and $\geq$25M for HalfCheetah, while the figure presents results for training with only 100K timesteps.
Furthermore, D2 achieves expert performance with fewer samples than DAC and recovers the expert policy on environment Walker2d and Swimmer. 
D2 also performs slightly better on Hopper and trains more stably than DAC. 
On HalfCheetah, D2 does not outperform DAC but both algorithms converge slowly with the given demonstrations. 
After comparing HalfCheetah demonstrations with the ones for other tasks, 
we found action values demonstrated by SAC policies in HalfCheetah task are mostly near $1.0$ or $-1.0$ (torque limits),
which could make it difficult for a policy network with $\tanh$ as the output activation to approximate.

We also vary the number of demonstration trajectories,
as shown in Figure~\ref{fig:comparison-num-traj}. 
The average performance is reported for each algorithm (10 repeats) after training with only 100K environment interactions. 
D2 consistently outperforms all other methods in Walker2d, Swimmer and Hopper with different numbers of demonstration trajectories,
with statistical significance on Walker2d, Swimmer and Hopper. 

We further evaluate D2 and DAC with 20 trajectories on two more domains, classic controls and Box2D. 
Results are presented in Table~\ref{tab:d2-vs-dac}. 
D2 significantly outperforms DAC after training only for 50K time-steps, and
also significantly faster than DAC in wall clock time 
(as D2 requires no training of any reward function). 

\subsection{Comparing With Off-policy Distribution Matching}\label{sec:exp-distribution-matching}
\begin{figure}
    \centering
    \begin{subfigure}[b]{0.48\linewidth}
        \centering
        \includegraphics[width=\linewidth]{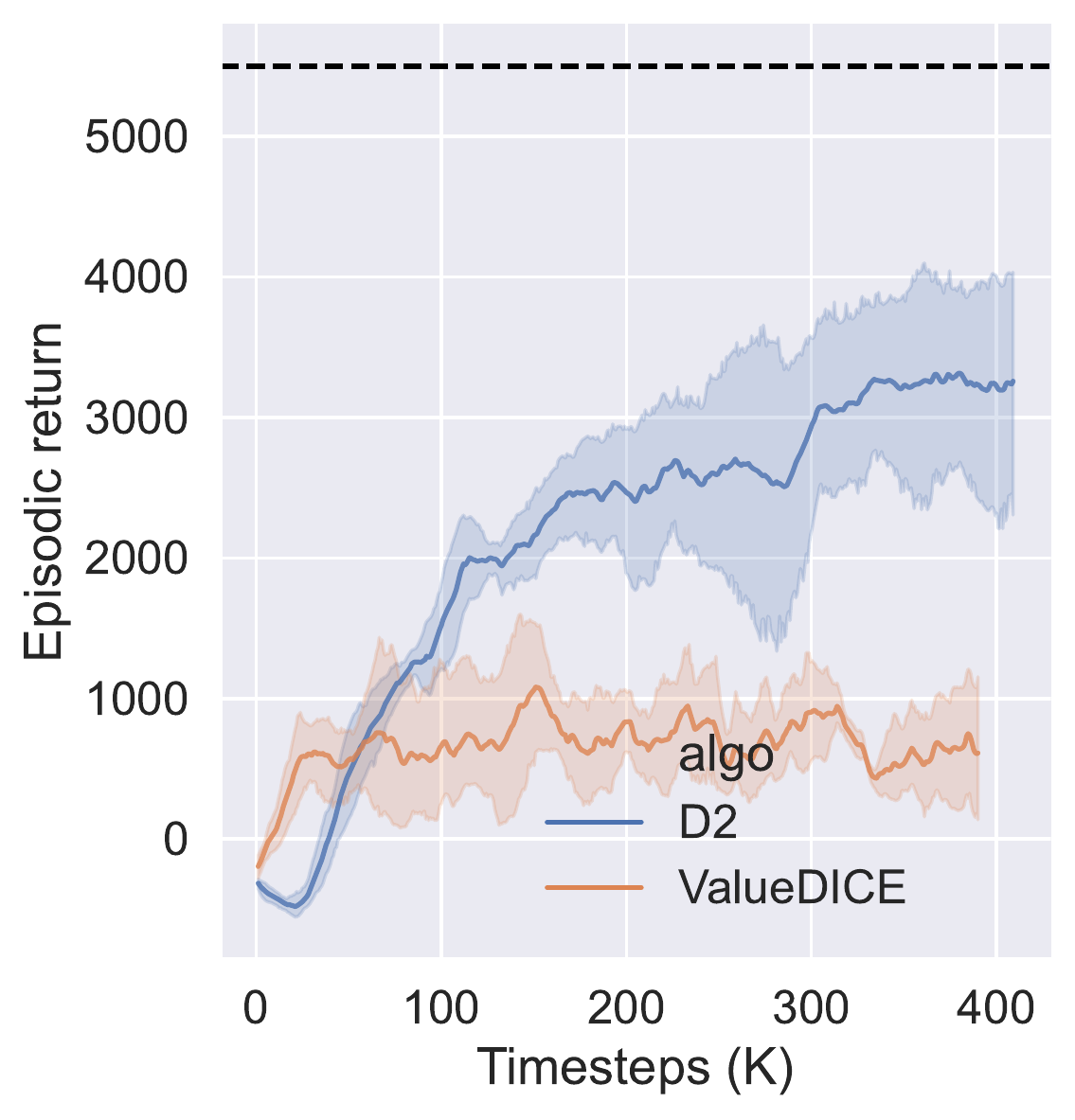}
        \caption{{\small DDPG demonstrations}}    
    \end{subfigure}
    \begin{subfigure}[b]{0.48\linewidth}
        \centering 
        \includegraphics[width=\linewidth]{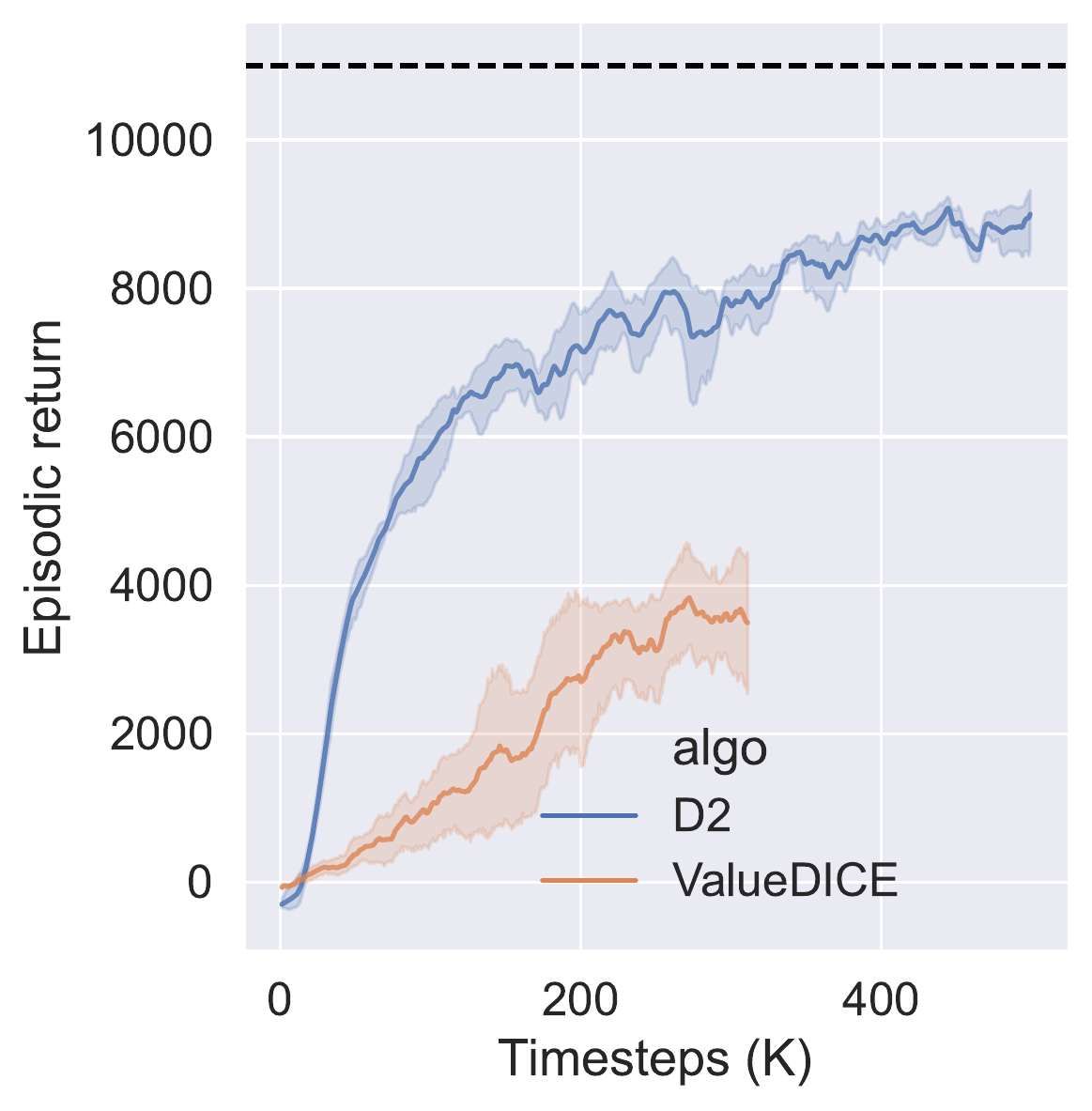}
        \caption{{\small SAC demonstrations}}    
    \end{subfigure}
    \begin{subfigure}[b]{0.48\linewidth}
        \centering 
        \includegraphics[width=\linewidth]{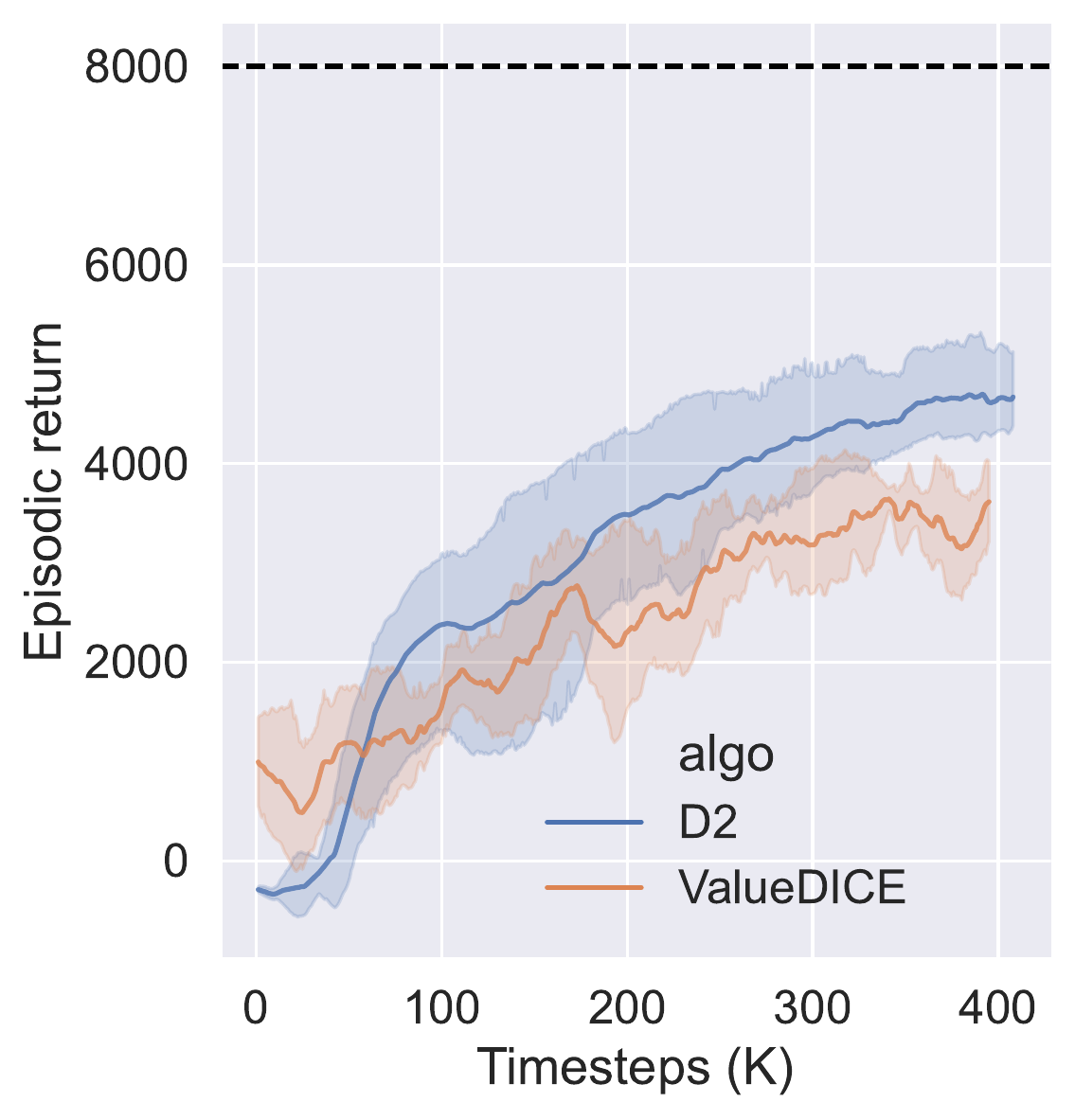}
        \caption{{\small TD3 demonstrations}}    
    \end{subfigure}
    \begin{subfigure}[b]{0.48\linewidth}
        \centering 
        \includegraphics[width=\linewidth]{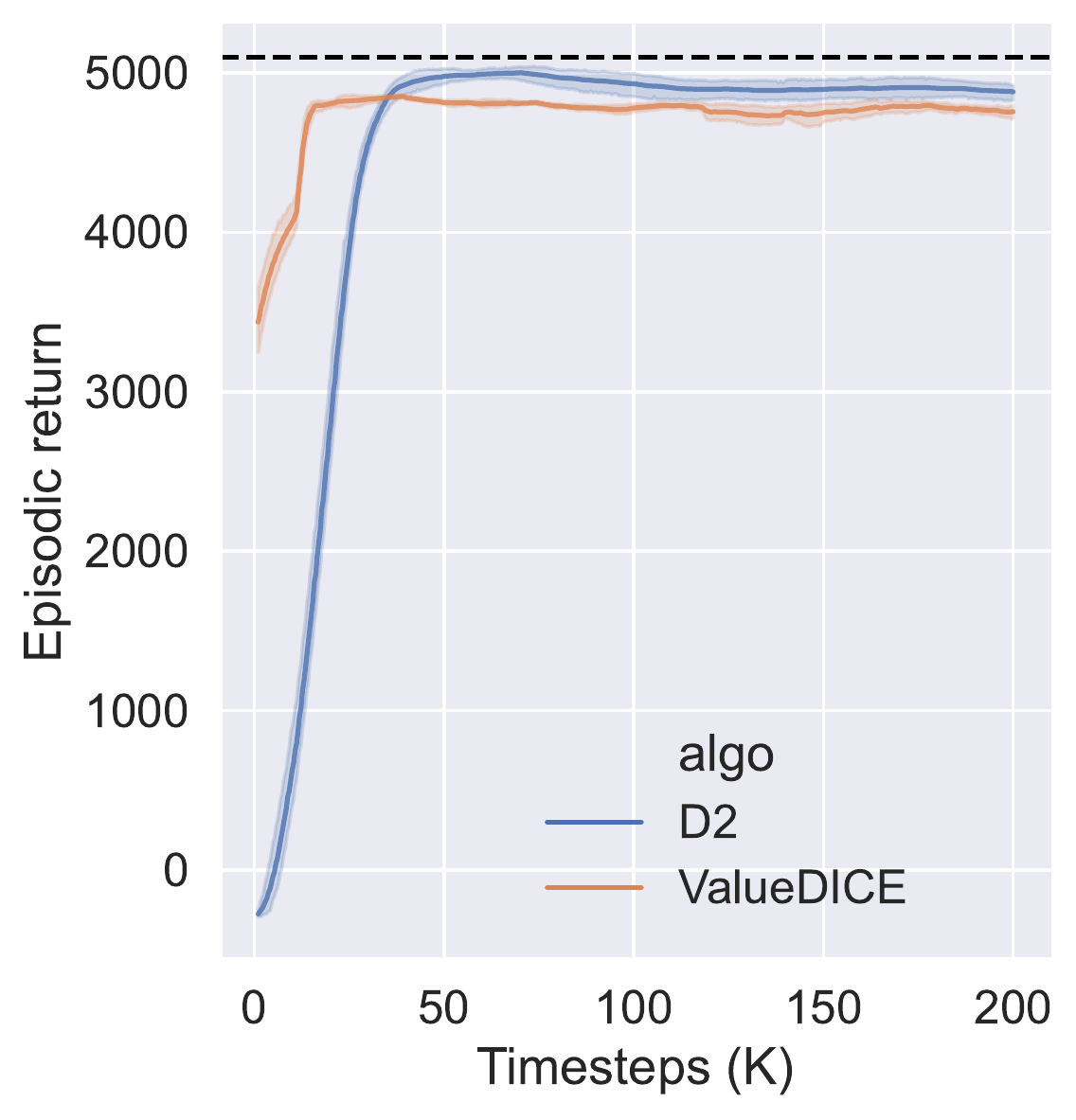}
        \caption{{\small TRPO demonstrations}}    
    \end{subfigure}
    \caption{Comparison with ValueDICE using demonstrations generated from different algorithms. The expert performance is also indicated by the black dashed line.}
    \label{fig:comparison-with-distribution-matching}
\end{figure}
Second, we compare D2 with ValueDICE~\cite{kostrikov2019imitation}, 
the state-of-the-art off-policy distribution matching approach for imitation learning.
ValueDICE leverages the change of variable trick to turn the on-policy distribution matching problem into an off-policy optimization problem. 
However, as pointed out by~\citet{sun2021softdice}, such change of variable implicitly assumes that the on-policy samples should be strictly sampled from a stationary state-action distribution. 
Otherwise, the off-policy optimization objective would be biased in practice. 
We thus evaluate D2-Imitation and ValueDICE on demonstrations generated by different RL algorithms.
Specifically, in addition to the expert trajectories generated in previous evaluation, 
we further use 20 trajectories by policies trained via TRPO, DDPG and TD3.\footnote{TRPO demonstrations are from \url{https://github.com/google-research/google-research/tree/master/value_dice}.}
Policies in TRPO and SAC are stochastic while those in DDPG and TD3 are deterministic. 
We then evaluate D2 and ValueDICE on these four types of demonstrations. 
Figure~\ref{fig:comparison-with-distribution-matching} presents the training performance of two algorithms on the HalfCheetah environment (we pick HalfCheetah as it was reported that learning the expert policy could take more than 20 million timesteps for imitation learning methods~\cite{kostrikov2018discriminator}).
Although ValueDICE achieves good sample efficiency with demonstrations given by TRPO (stochastic policies), 
its performance becomes worse when demonstrations change. 
Empirically, ValueDICE is sensitive to demonstrations, and can fail to learn the expert policy. 
For example, ValueDICE initially progresses on DDPG demonstrations, then plateaus after 50K time-steps and never reaches expert performance.
Furthermore, we found that even with demonstrations that are generated by stochastic policies, e.g., SAC, D2 still outperforms ValueDICE. 
Intuitively, this could be because a stochastic policy can always be turned into deterministic by picking the best or mean action,
the discriminator in D2 is trained to predict that action for each state,
and off-policy TD learns a deterministic policy that distills such action information from replay buffers.  
Overall, compared with ValueDICE, the performance of D2 imitation is more consistent across different types of demonstrations. 

\begin{table}
\begin{tabular}{ccccc}
\toprule
 & IP & MCC & BW & LLC \\\midrule
 & \multicolumn{4}{c}{\textbf{Returns for 50K training timesteps}}    \\\midrule
Expert PPO & 1000.0 & -0.050 & 302.49  &  181.21\\\midrule
BC & 134.2 & -0.159 & -106.30   & -75.62 \\\midrule
DAC & \textbf{1000.0} & -0.085 & 297.58 & -73.49\\\midrule
D2 & \textbf{1000.0} &  \textbf{-0.045}\tiny{**} &  \textbf{308.32} & \textbf{-34.63}\tiny{**} \\\midrule
 & \multicolumn{4}{c}{\textbf{Wall clock time for 50K timesteps (s)}} \\\midrule
DAC & 2755.0 & 2937.4 & 2458.8 & 2708.4  \\\midrule
D2 & \textbf{1526.9}\tiny{**}  & \textbf{1540.9}\tiny{**} &  \textbf{1720.2}\tiny{**} & \textbf{1586.2}\tiny{**} \\\bottomrule
\end{tabular}
\caption{D2 vs DAC on InvertedPendulum (IP), MountainCarContinuous (MCC), 
BipedalWalker (BW) and LunarLanderContinuous (LLC); $**$ for significance ($t$-test).}
\label{tab:d2-vs-dac}
\end{table}

\subsection{Ablation Studies}

The discriminator is crucial for D2-Imitation to work properly and guarantees convergence to  expert performance. 
We perform ablations on the discriminator and compare the training performance of D2-Imitation with one variant: 
one without any discriminator (denoted \textit{without-discriminator}), 
which just puts on-policy samples to $\mathcal{B}^{0}$, assigns $0$ reward to them and gives $+1$ reward to demonstration samples in $\mathcal{B}^{+}$, an idea adopted in Soft-Q Imitation Learning (SQIL)~\cite{reddy2019sqil}.
Figure~\ref{fig:ablation-discriminator} shows that training quickly plateaus or even collapses if no discriminator is applied. 
Furthermore, this plateauing effect happens more quickly when fewer demonstration trajectories are used. 
This phenomenon has also been reported in SQIL. 
To avoid such training collapse, SQIL requires early stopping of the training process by judging whether the squared soft Bellman error converges to a minimum, 
which we argue can be challenging as this early stopping also relates to the number of trajectories used in training. 
Moreover, even with perfect early stopping (say 100K interactions as in the given figure in which the true reward function is known), 
the policy still fails to achieve the expert performance.

\begin{figure}
    \centering
    \includegraphics[width=0.9\linewidth]{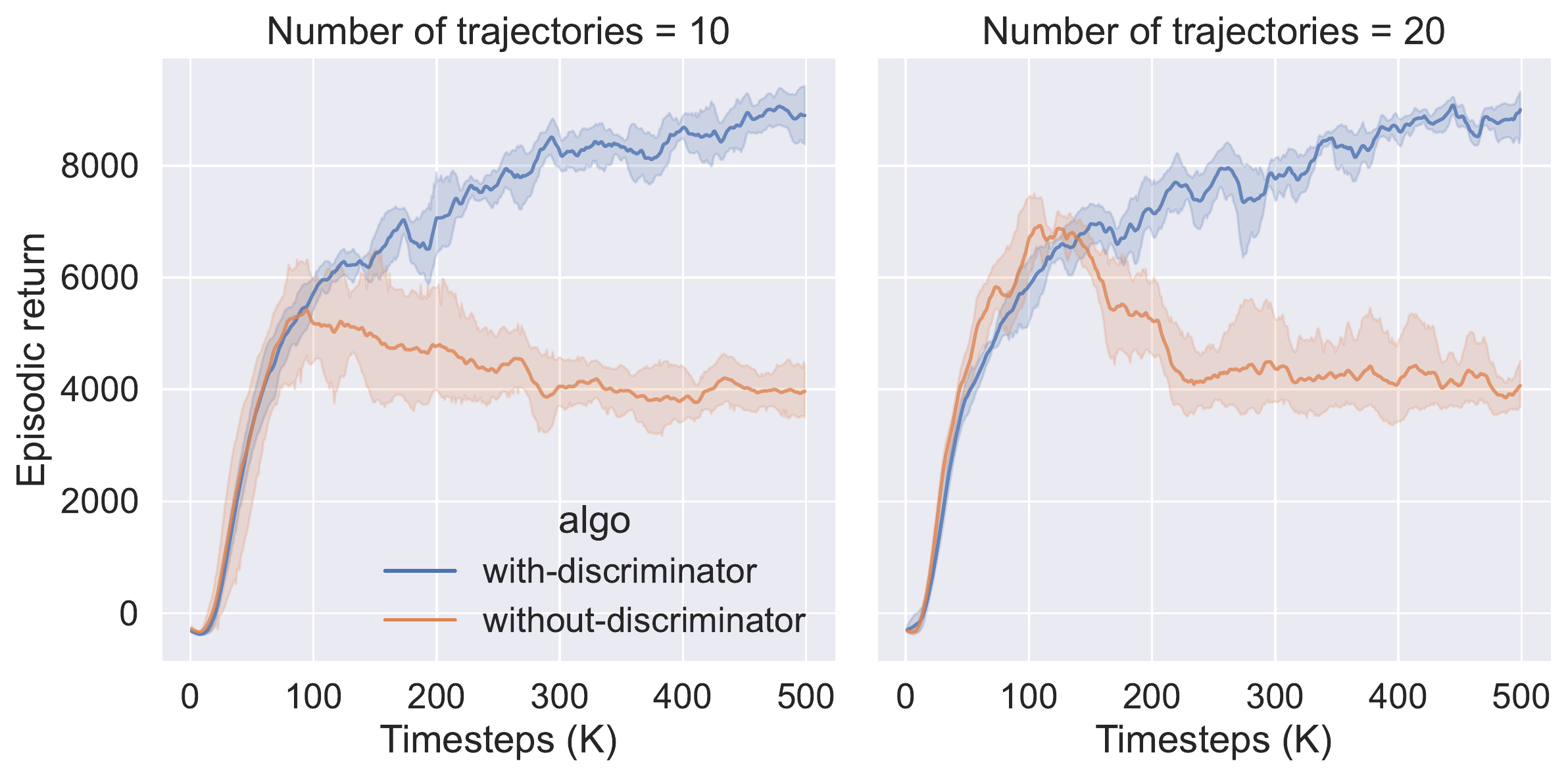}
    \caption{Effect of the discriminator on the training performance with 20 and 10 SAC demonstration trajectories.}
    \label{fig:ablation-discriminator}
\end{figure}

\section{Related Work}
Adversarial imitation casts imitation learning as a distribution matching problem~\cite{ho2016generative} and leverages GANs~\cite{goodfellow2014generative} to minimize the Jensen-Shannon divergence between distributions induced by the expert and the learning policy. 
This approach avoids the difficulty of learning reward functions directly from demonstrations but is generally sample intensive.
To improve sample efficiency, many methods extend adversarial imitation to be off-policy. 
For instance, Discriminator-actor-critic (DAC)~\cite{kostrikov2018discriminator,sasaki2018sample} improves sample efficiency by reusing previous samples stored in a relay buffer.
However, this approach still relies on non-stationary reward signals generated by the discriminator, which can make the critic estimation hard and training unstable. 
Recent work proposes to train a fixed reward function by estimating the support of demonstrations and training  critics with the fixed reward~\cite{wang2019random}.
This support estimation itself could be hard given that only a limited number of empirical samples are available from the considered distributions. 
Another line of off-policy distribution matching approaches focuses on estimating the critics directly without learning any reinforcement signals~\cite{sasaki2018sample,kostrikov2019imitation}.
The state-of-the-art along this line is  ValueDICE~\cite{kostrikov2019imitation}, which casts distribution matching as off-policy density ratio estimation and updates the policy directly via a max-min optimization. 
However, as we show in the analysis and experiments, these methods can be ill-posed when the demonstrations are generated from deterministic policies, and their performance can also be sensitive to the demonstrations used in training.

Recently, some non-adversarial imitation learning approaches have been proposed. 
For example, offline non-adversarial imitation learning~\cite{arenz2020non} reduces the min-max in ValueDICE to policy iteration, 
which however still requires estimating density ratios and can potentially inherit the same issues from ValueDICE.
By contrast, D2-Imitation avoids density ratio estimation and is more robust to different demonstrations. 
Primal Wasserstein imitation learning~\cite{dadashi2020primal} avoids adversarial training by introducing an off-line cost function, 
which, however, requires a domain-specific metric and can be challenging to properly specify for different environments.
D2-Imitation does not need such domain knowledge and can be applied in many different settings.  
Our reward design for D2-Imitation looks superficially similar to that of Soft $Q$ Imitation Learning (SQIL), 
which assigns $+1$ reward to demonstration and zero for all interaction samples~\cite{reddy2019sqil}. 
However, D2-Imitation is fundamentally different in that the reward assignment is theoretically consistent with the use of deterministic policies.

\section{Conclusion}
In this paper, we revisited the foundations of adversarial imitation. 
We leveraged the similarity between the Bellman equation and the stationarity equation to derive a TD learning approach, which directly learns a special proxy, i.e., basic feasible solutions, for the expert state-action distribution. 
Moreover, we showed that the use of deterministic policies simplifies TD learning and yields a practical learning algorithm, 
D2-Imitation, which operates by first partitioning samples into two replay buffers and then learning a deterministic policy via off-policy deterministic policy gradients. 
Finally, the notion of partitioning samples into two groups theoretically follows from the use of a deterministic policy. 
Our empirical results demonstrated that D2-Imitation is effective in achieving good sample efficiency, and outperforms many adversarial imitation approaches on different control benchmark tasks with demonstrations generated by either deterministic or stochastic policies. 
Also, D2-Imitation consistently outperforms the state-of-the-art off-policy distribution matching method when training with various different types of demonstrations. 
In conclusion, D2-Imitation, as a direct result of leveraging two novel insights in the distribution matching formulation, is a simple yet very effective sample-efficient imitation learning approach. 

\section*{Acknowledgements}
Mingfei Sun is partially supported by funding from Microsoft Research. 
The experiments were made possible by a generous equipment grant from NVIDIA. 

\bibliography{ref}

\end{document}